\documentclass[anon,12pt]{colt2020} 

\title[]{Degenerated Noisy Matrix Completion}
\usepackage{times}
\usepackage[noend]{algorithm,algorithmic}
\usepackage{caption}
\usepackage{amsmath, amssymb}
\usepackage[utf8x]{inputenc} 
\usepackage{hyperref}
\usepackage{interval}
\usepackage{caption}
\usepackage{hyperref}
\usepackage{cleveref}%

\usepackage{natbib}

\usepackage{tikz}

\definecolor{darkgray}{rgb}{0.50, 0.50, 0.50}
\definecolor{gray}{rgb}{0.70, 0.70, 0.70}
\definecolor{lightgray}{rgb}{0.92, 0.92, 0.92}

\coltauthor{%
 \Name{Jafar Jafarov} \Email{jafarov@alumni.stanford.edu}\\
 \addr Stanford University%
}

\begin{document}

\hspace{15mm} \LARGE Matrix Completion with Sparse Noisy Rows  \\[2ex]

\large 

\textbf{Jafar Jafarov} \hspace{60mm} jafarov@alumni.stanford.edu

\normalsize

Stanford University

\begin{abstract}%
Exact matrix completion and low rank matrix estimation problems has been studied in different underlying conditions.
In this work we study exact low-rank completion under non-degenerate noise model.
Non-degenerate random noise model has been previously studied by many researchers under given condition that the noise is sparse and existing in some of the columns.
In this paper, we assume that each row can receive random noise instead of columns and propose an interactive algorithm that is robust to this noise. We show that we use a parametrization technique to give a condition when the underlying matrix could be recoverable
and suggest an algorithm which recovers the underlying matrix.
\end{abstract}

\vspace{8mm}
\section{Introduction}
\vspace{3mm}

Noisy matrix completion has been remaining interesting topic for long time.
The problem has been inspired by after realizing many real life datasets are obtained by an underlying low-rank matrix composed with a noise source. 
Problem setting varies many times due to classification of noise. 
Generic framework would describe the problem is:
\begin{align*}
    N = M + \mathcal{E}
\end{align*}
where $N$ is an $n_1 \times n_2$ sized matrix with any rank depending on $M$ and $\mathcal{E}$. 
$M$ is an $n_1 \times n_2$ sized rank$-r$ matrix, while $\mathcal{E}$ is a matrix that represents the noise.
The classification of the problem here is due different type of noise matrices.\\[2ex]
One setting of the noise is small bounded noise that is applied to entire matrix. 
In this setting we target to recover the full-rank matrix by estimating its column space.
This problem has been studied heavily under the topic low-rank estimation where it has been shown the column space due to highest singular values of the matrix is the best estimation of the matrix.\\[2ex]
Another setting this problem has been studied heavily is under sparse noise coming from non-degenerate space. 
There are multiple categories that this problem could be studied under this assumption.
\cite{nina} and \cite{ramazanliadaptive} has studied this problem under the condition that sparse set of columns has non-degenerate noise.
We will talk about it further later in this section. \\[2ex]
Another sparse noise model is some of the few entries in the matrix getting noise (i.e. few entries in $\mathcal{E}$ is non-zero). 
This type of noise would model the condition that the communication channel can corrupt any entry in the matrix.
Finally, the noise model we study in this paper is sparse random noise in few of the rows.\\[2ex]
Matrix Completion problem has been studied in several different categories in terms of the learning algorithm as well.
There has been many adaptive learning algorithms has been applied to solve this problem.
\cite{akshay1, akshay2, nina, ilqarsingle, ramazanli2022adaptive, ramazanli2022matrix} are few of the examples in this space.
Another approach in this space is due to passive algorithms.\\[2ex]
The difference between these class of algorithms are in adaptive setting learning process is interactive. 
At every step set of available information changes dynamically, and we can decide what kind of information receive next based on the existing information set.
However, in the passive setting information is available prior to learning process and we try to learn as accurate as possible from this information set.\\[2ex]
This problem has been studied in detail, and it has been shown that the nuclear norm minimization approach can solve it which has been discussed in \cite{recht1}.
Noisy-matrix completion has been studied in adaptive setting as well, in the paper \cite{nina} it has been shown that to study the upper bound for the angle between column space of B and column space of A can give a good estimation can be translated to a low-rank estimation algorithm.\\[2ex]
In this paper, we focus on adaptive matrix completion problem with non-degenerate noise.
Non-degenerate noise has been studied well in \cite{nina} and \cite{ilqarsingle} but for both of these papers it has been assumed that the noise appears to be sparse and occur only in few of the columns.
However, it still has still been an open question what would have happened if noise occurs in many columns instead of few.\\[2ex]
Low-rank estimation has many applications for real life problems which has been discussed in many previous works.  
There has been multiple different direction of ideas has been used to approeach this problem. 
In passive setting more popular option has been nuclear norm minimization which has been studied in \cite{tao, recht2}.
In adaptive setting, researchers has used adaptive column sampling using coherence of column space \cite{akshay1, akshay2, nina}.\\[2ex]
It is an interesting open question how to approach heavily non-degenerate noisy matrix completion in the passive setting. 
We leave that question as conjecture at this point and will focus on adaptive methods within this paper.
Adaptivity have been shown boosting the performance in many machine learning problems.
The power of adaptive sampling had been illustrated even earlier than \cite{akshay1}.
\citep{haupt, malloy, balak, castro}
proved adaptivity outperforms passive schemes.

\newpage
\section{Preliminaries}
\label{sec:preliminaries}
\vspace{3mm}

In this paper we represent the received noisy matrix as input to out algorithm $N$. 
It is assumed that the matrix has size of $n_1\times n_2$, and it has an underlying low-rank structured matrix $M$, which added non-degenerate noise to some row vectors.\\[2ex]
Characteristics of non-degenerate spaces has been discussed both \cite{nina} and \cite{ilqarsingle}.
Lets visit the description of the non-degenerate space here as well.
Given that $\mathbf{E}^s \in \mathbb{R}^{s\times n_2}$ consists of s-many corrupted vectors coming from a non-degenerate distribution then we have the following conditions satisfied with probability of 1 as cited in \cite{nina}:\\[2ex]
\textbf{1:}\: $\mathrm{rank}(\mathbf{E}^s)= s$ for any $s \leq n_2$\\[2ex]
\textbf{2:}\: $\mathrm{rank}(\mathbf{E}^s,x)= s+1$ holds for $x\in U^k\subset \mathbb{R^m}$ uniformly and $s \leq n_2-k$, where $x$ can be depend or independent on $\mathbf{E}^s$\\[2ex]
\textbf{3:}\: $\mathrm{rank}(\mathbf{E}^s,U^k)= s+k$ given that $s+k\leq n_2$\\[2ex]
\textbf{4:}\: The marginal of non-degenerate distribution is non-degenerate \\[2ex]
Interesting implication of the definition of non-degenerated noisy vector, is that when we delete them from the set of rows, the rank of the remaining matrix will decrease by one. 
This is an important property, as we are using it to detect the noisy row in the given matrix.\\[2ex]
Through the paper we denote the set $\{1, 2, \ldots, n\}$ by the notation $[n]$.
We use the same notation as \cite{poczos2020optimal} for the sparsity number, as $\psi()$ stands for the that takes matrix or subspace as input and return and integer that represents the sparsity number. 
The sparsity number of subspace $\mathcal{X} \subset \mathbb{R}^m$  defined as:
\begin{align*}
  \psi(\mathcal{X})=\mathrm{min}\{\|x\|_0 | x\in \mathcal{X} \text{ and } x\neq 0 \}  
\end{align*}
and sparsity-number of a matrix is just simply sparsity-number of its column space.
Similarly, we define the \textit{nonsparsity-number} of the subspace $\mathcal{X} \subset \mathbb{R}^n_1$   as: 

\begin{align*}
 \overline{\psi}(\mathcal{X})= n_1 - \psi(\mathcal{X})
\end{align*}

\vspace{3mm}
\hspace{-7mm} $N_{ij}$ stands for the entry of the matrix that is placed in the $i$-th row and $j$-th column of the matrix $N$.
We also use $N_{i:j}$ for the representation of $N_{ij}$. Moreover, $N_{i:}$ stands for the $i-$th row of the matrix $N$ and $N_{:j}$ stands for the $j$-th column of the matrix $N$.
For a set $\Pi \subset [n_1]$, the sub-matrix induced by the rows from $\Pi$ is donated by $N_{\Pi:}$.
Similarly, for a set $\Gamma \subset [n_1]$, the submatrix that is induced by the columns from $\Gamma$ is given by $N_{:\Gamma}$.
Finally, the $|\Pi| \times |\Gamma|$ sized submatrix that is induced from the rows in $\Pi$ and columns $\Gamma$ is represented by $N_{\Pi:\Gamma}$.

\newpage
\section{Main Results}
\vspace{3mm}

In this section we describe the main result of the paper.
Previously matrix completion has been studied with degenerated noise under in a sparse set of columns. 
\cite{nina} and \cite{ilqarsingle} proposed different algorithms to recover the matrix 
\vspace{2mm}
\begin{align*}
    N = M + \Delta 
\end{align*}
where $M$ is an $n_1 \times n_2$ sized rank-$r$ matrix, and $\Delta$ is an $n_1 \times n_2$ sized matrix, which has entries only in some of the columns, $\Pi \subset [n_2]$:
\vspace{2mm}
\begin{align*}
    \Delta_{ij} = 
    \begin{cases}
      \delta \hspace{5mm} \text{ where } \hspace{1mm} \delta \sim \mathbf{E} \text{ nondegenerate distribution}  \hspace{20mm} \text{ if } j \in \Pi \\
      0 \hspace{89mm} \text{ if } j \in [n_2] \setminus \Pi
    \end{cases}
\end{align*}
In this paper, we will approach matrix completion under different class of noise. 
Specifically, we assume that the non-degenerate noise is added to sparse set of rows, rather than columns:
\vspace{2mm}
\begin{align*}
    N = M + \Delta 
\end{align*}
where for a $\Gamma \subset [n_1]$ we have:
\vspace{2mm}
\begin{align*}
    \Delta_{ij} = 
    \begin{cases}
      \delta \hspace{5mm} \text{ where } \hspace{1mm} \delta \sim \mathbf{E} \text{ nondegenerate distribution}  \hspace{20mm} \text{ if } i \in \Gamma \\
      0 \hspace{89mm} \text{ if } i \in [n_1] \setminus \Gamma
    \end{cases}
\end{align*}
Moreover, we assume that the space sparsity number column space of $M_{2}$ is larger than 1, i.e. $\psi(U)>1$.
This condition is necessary in the recovery process, because if $\psi(U)=1$, then it means that, there exists an $i_0$, which satisfies $e_{i_0} \in U$.
Therefore, deleting the row with corresponding index to $i_0$ will cause to the reduction in the rank.
Which this will make the row indistinguishable from a row that is a complete noise.
Therefore, it is necessary to have pre-condition $\psi(U) > 1$.\\[2ex]
The idea of the algorithm below we provide is as following. 
In the first phase of the algorithm, we are applying exact same technique that has been applied in \cite{ilqarsingle}.
We target to detect as many possible linearly independent columns and linearly independent rows.\\[2ex]
In the aforementioned paper authors achieve this goal by simultaneously studying both row and column spaces.
The idea of the algorithm, is due to following fact.
Authors shows that for a set of linearly independent rows $R$ of the underlying matrix $M$ (i.e. $M_{R:}$) and for a set of linearly independent columns $C$, (i.e. $M_{:C}$), the induced submatrix $M_{R:C}$ has also linearly independent rows and columns.\\[2ex]
Moreover, it also has been shown that if for an index pair of $M_{ab}$, satisfies the condition that $M_{R'C'}$ has rank of $r+1$, given that $R'$ is the set $R \cup \{a\}$ and $C'$ is the set $C \cup \{ b \}$ then we can conclude that the induced submatrix $M_{R':}$ has $r+1$ linearly independent row vectors and $M_{:C'}$ has $r+1$ linearly independent column vectors.\\[2ex]
\begin{algorithmic}[1]
    \STATE   $ \zeta = 0$
    \STATE   $ \eta = \mathrm{max} (\frac{2n_1}{n_2}\log{\frac{1}{\epsilon}}, \: \log{\frac{1}{\epsilon}} ) $    
    \WHILE{$ \zeta < \eta $}
    \STATE   $ \zeta \leftarrow \zeta + 1$
    \FOR{$j$ from $1$ to $n_2$}
    \STATE Query $\mathbf{M}_{i:j}$ for random $i$  
    \STATE $\widehat{R} \leftarrow R \cup \{i\}$ 
    \STATE $\widehat{C} \leftarrow C \cup \{j\}$  
    \STATE \textbf{If} $ \mathbf{M}_{\widehat{R}:\widehat{C}}$ is convertiblle :
    \STATE \hspace{0.1in} Query $\mathbf{M}_{:j}$ and $\mathbf{M}_{i:} $ 
    \STATE \hspace{0.1in} $R \leftarrow \widehat{R}$ 
    \STATE \hspace{0.1in} $C \leftarrow \widehat{C}$ 
    \STATE \hspace{0.1in} $r \leftarrow \widehat{r} + 1$ 
    \STATE \hspace{0.1in} $\zeta \leftarrow 0$     
    \ENDFOR

\ENDWHILE

\STATE Identify $\mathcal{E}$
\STATE Orthogonalize column vectors in $C$ in induced submatrix $M_{R\setminus \mathcal{E} : C}$

\FOR{each column $j \in [n_2]\setminus C$ }
\STATE Recover $\widehat{\mathbf{M}}_{:j}$ 
\ENDFOR
\end{algorithmic}

\vspace{6mm}
\hspace{-6mm}The main observation here is to notice that given some of the rows are complete noise, then column space is guaranteed to contain some of the standard basis vectors.
The reason behind this is that, given that a row is a complete noise, then deleting this row will reduce the rank of the matrix.
This can only happen if one standard basis vector is contained in the column space.
The reason behind this is as simple as following observation.\\[2ex]
Assuming that deleting a row, row-$i$ reduces the rank of the matrix. 
Then it simply means that the number of linearly independent columns in the submatrix $M_{R\setminus \{i\}: }$ is smaller than number of linearly independent column in the submatrix $M_{R: }$.
Then, there exists a set of non-trivial coefficients $\gamma_1, \gamma_2, \ldots, \gamma_t$ such that:

\begin{align*}
    \gamma_1 M_{R\setminus \{i\}:c_1 } + \gamma_2 M_{R\setminus \{i\}:c_2 }  + \ldots + \gamma_r M_{R\setminus \{i\}:c_t } = 0
\end{align*}

\hspace{-6mm} for some column indices $c_1, c_2, \ldots, c_t$. However, 

\begin{align*}
    \gamma_1 M_{R:c_1 } + \gamma_2 M_{R:c_2 }  + \ldots + \gamma_r M_{ R:c_t } \neq 0
\end{align*}

\hspace{-6mm}  then it simply follows that:

\begin{align*}
    \gamma_1 M_{R:c_1 } + \gamma_2 M_{R:c_2 }  + \ldots + \gamma_r M_{ R:c_t } = e_i
\end{align*}

\hspace{-6mm}  which simply implies that $e_i$ is contained in the column space of $M$.

\begin{theorem}
Given $N$ is an $n_1 \times n_2$ sized matrix with random noise in some subset of rows $\Omega$.
Given that the induced submatrix $N_{R\setminus \Omega :}$ has row space $U$ and column space $V$ with the condition that $\psi(U) > 1$.
Then with the probability of $\: 1-2 \epsilon \:$ the algorithm above successfully identifies rows with noise, and recovers remaining entries using at most
\begin{align*}
(n_1+n_2- | \Omega |) |\Omega| +  \frac{\frac{4n_1}{\psi(U)}(r+2 +\log{\frac{1}{\epsilon}})}{\psi(V)}n + 2n_1( | \Omega | +2+\log{\frac{1}{\epsilon}})
\end{align*}
observations.
\end{theorem}

\begin{proof}
The proof of the algorithm has a very similar flavor to the proof of $\mathbf{ERR}$ in \cite{ilqarsingle}.
The key point to notice here is the success probability of finding a row that is noise is still at least $\frac{ \psi(U) }{m}$. 
The reason behind this is that, assuming that there are $| \Omega |$ many rows those are noisy, and if we observe an entry from any of them we are guaranteed  notice the linear independence.\\[2ex]
However, if we observe any entry from points those are not noisy for the lemma 8 in the paper \cite{poczos2020optimal} we have that the probability of detection of linear independence is simply $\frac{\psi(U) }{n_2-|\Omega|}$.
Therefore,  the probability of detection of any new column is simply:

\begin{align*}
    \frac{|\Omega|}{n_1} + \frac{n_1-|\Omega|}{n_1} \frac{ \psi(U) }{n_1-|\Omega|} =
    \frac{|\Omega|+\psi(U)}{n_1}
\end{align*}
which is lower bounded by $\frac{\psi(U)}{n_1}$. An interesting observation here is due to \cite{poczos2020optimal} which simply states that after

\begin{align*}
    \frac{2n_1}{\psi(U)}\big(r+|\Omega|+ 2 +\log{\frac{1}{\epsilon}}\big)    
\end{align*}
observations, we detect all $|\Omega|+r$ many linearly independent rows with probability of $1-\epsilon$.
Here we can observe that, once we detect all non-noisy rows, the probability of detection of noisy rows reduces.
But this reduced probability will always be lower bounded by $1/n_1$. 
Therefore, similar to the argument above, after observing
\begin{align*}
    2n_1(|\Omega|+2+\log{\frac{1}{\epsilon}})
\end{align*}
entries with probability of $1-\epsilon \:$ all the noisy rows will be detected.
Therefore, using the union bound argument, we can conclude that, with probability $1-2 \epsilon$ after observing
\begin{align*}
    2n_1(|\Omega|+2+\log{\frac{1}{\epsilon}}) + \frac{2n_1}{\psi(U)}\big(r+|\Omega|+ 2 +\log{\frac{1}{\epsilon}}\big)    
\end{align*}
entries, we detect $r$ many linearly independent rows, and also we are able to detect noisy rows.
Which concludes the statement of the theorem.

\end{proof}

\bibliography{ms}



\end{document}